\documentclass[
twocolumn,
]{ceurart}

\usepackage{listings}
\usepackage{amsmath, amssymb, amsthm}
\usepackage{mathtools, bm, bbm, thm-restate}
\usepackage[skip=2pt]{caption}
\usepackage[skip=2pt]{subcaption}
\usepackage{microtype, nicefrac}
\usepackage{graphicx, wrapfig}
\usepackage{booktabs, adjustbox, multirow}
\usepackage{xcolor, colortbl}
\usepackage{tikz, float}
\usetikzlibrary{shapes,arrows}
\usepackage{threeparttable}
\usepackage[capitalize,noabbrev]{cleveref}

\input{math_commands.tex}

\theoremstyle{plain}

\theoremstyle{definition}

\theoremstyle{remark}

\newcommand{\norm}[1]{\left\lVert#1\right\rVert}

\newcommand{\set}[1]{\mathcal{#1}}
\newcommand{\prob}{\mathbb{P}}

\newcommand\eqdef{\stackrel{\mathclap{\scriptsize\mbox{def}}}{=}}
\newcommand{\gray}[1]{\textcolor{gray}{#1}}

\begin{document}

\copyrightyear{2023}
\copyrightclause{Copyright for this paper by its authors.
  Use permitted under Creative Commons License Attribution 4.0
  International (CC BY 4.0).}

\conference{AISafety-SafeRL 2023 Workshop (IJCAI), August 19–21, 2023, Macao, SAR, China}

\title{Diffusion Denoised Smoothing for Certified and Adversarial Robust Out-Of-Distribution Detection}

\tnotemark[1]
\tnotetext[1]{This  work  was  funded  by  the  Bavarian  Ministry  for  Economic  Affairs,  Regional  Development  and  Energy.}

\author[1]{Nicola Franco}[%
email=nicola.franco@iks.fraunhofer.de,
]
\cormark[1]
\address[1]{Fraunhofer Institute for Cognitive Systems IKS, Munich, Germany}
\address[2]{Dept. of Computer Science \& Munich Data Science Institute, Technical Univ. of Munich, Germany}

\author[1]{Daniel Korth}[email=daniel.korth@iks.fraunhofer.de]
\author[1]{Jeanette Miriam Lorenz}[email=jeanette.miriam.lorenz@iks.fraunhofer.de]
\author[1]{Karsten Roscher}[email=karsten.roscher@iks.fraunhofer.de]
\author[2]{Stephan G{\"u}nnemann}[email=s.guennemann@tum.de]

\cortext[1]{Corresponding author.}

\begin{abstract}
    As the use of machine learning continues to expand, the importance of ensuring its safety cannot be overstated. 
    A key concern in this regard is the ability to identify whether a given sample is from the training distribution, or is an "Out-Of-Distribution" (OOD) sample.
    In addition, adversaries can manipulate OOD samples in ways that lead a classifier to make a confident prediction.
    In this study, we present a novel approach for certifying the robustness of OOD detection within a $\ell_2$-norm around the input, regardless of network architecture and without the need for specific components or additional training.
    Further, we improve current techniques for detecting adversarial attacks on OOD samples, while providing high levels of certified and adversarial robustness on in-distribution samples.
    The average of all OOD detection metrics on CIFAR10/100 shows an increase of $\sim 13 \% / 5\%$ relative to previous approaches.
    Code: \url{https://github.com/FraunhoferIKS/distro}
\end{abstract}


\maketitle

\section{Introduction \& Related Work}\label{intro}

Although recent advances in Machine Learning (ML) demonstrate its validity in a wide range of applications, its use in safety-critical conditions remains challenging.
Since the appearance of unexpected low robustness to natural~\citep{hendrycks2018benchmarking} and adversarial~\citep{pgd} perturbations to the input data, several types of defenses have been proposed along the years.
Two main branches of defenses exist: \textit{empirical}~\cite{madry2018towards} and \textit{certified}~\citep{randomized_smoothing}, which aim at \textit{improving} or \textit{assuring} the robustness of the prediction in the vicinity of the input, respectively.
\textit{Certified} defenses might give the inaccurate impression that robustness makes ML systems ready for deployment in safety-critical applications.
Unfortunately, further issues lie also beyond robustness, including the lack of guarantees for Out-Of-Distribution (OOD) data, the lack of fairness, or the lack of explainability~\cite{paleyes2022challenges}.

\textbf{OOD Detection.}
With Maximum Softmax Probability (MSP)~\cite{msp} as a baseline method, OOD detection aims to identify inputs that fall outside the scope of the training distribution.
Outlier Exposure (OE)~\cite{oe} trains models to differentiate between in-distribution (ID) and out-of-distribution (OOD) samples. 
Recent approaches include Virtual Outlier Synthesis (VOS)~\cite{vos} and LogitNorm~\cite{logitnorm}. 
VOS adaptively synthesizes virtual outliers, while LogitNorm normalizes the logit vector to reduce overconfidence, using thresholding for OOD detection.

\textbf{Adversarial OOD Detection.}
Other lines of research~\cite{acet, ccu, atom}, focus on providing low confidence for OOD data when perturbed with adversarial noise.
\citet{acet} show that ReLU networks can have arbitrarily high confidence for data that is \textit{far enough} from the training distribution. 
Additionally, they propose ACET~\cite{acet}, an adversarial training method to enforce low confidence on OOD data, but at the cost of decreased ID accuracy.
ATOM~\cite{atom} addresses this issue by using outlier mining techniques to automatically select a diverse set of OOD samples from a large pool of potential OOD samples.

\textbf{Guaranteed OOD Detection.}
Recent studies like \citet{good, prood} bring forth $\ell_\infty$-norm certified robustness for OOD data with a simple but effective method: Interval Bound Propagation (IBP)~\citep{ibp}.
GOOD~\cite{good} proposes a training approach using IBP, but it produce loose bounds, impacting accuracy. 
While ProoD~\cite{prood} combines a certified discriminator and OE model, achieving state-of-the-art performance but with practical limitations: low certified accuracy, reliance on external datasets, and reduced scalability due to IBP's impact on larger models.

\begin{table*}[htb] 
\vspace{-0.5em}
    \centering
    \caption{Comparison between this work and previous methods in terms of ID and OOD robustness properties. In this case, the \checkmark indicates that property was provided in the work. While (\checkmark) indicates that the property is actually lower than expected.}
    \label{tab:contribution}
    \begin{adjustbox}{width=0.8\textwidth,center}
        \begin{tabular}{lcccccccc}
            \toprule
             \multirow{3}{*}{Methods}   &\multicolumn{3}{c}{In-Distribution (ID) Accuracy} &\multicolumn{5}{c}{Out-Of-Distribution (OOD) Detection} \\
             \cmidrule(lr){2-4} \cmidrule(lr){5-9}
             &Clean &Adversarial  &Certified &Clean &Adversarial &\multicolumn{2}{c}{Certified} &Asymptotic \\
             & &$\ell_\infty$ &$\ell_2$ & &$\ell_\infty$ &$\ell_\infty$ &$\ell_2$ &underconfidence \\
            \midrule
            - \textbf{Standard} \\
            OE~\cite{oe}        &\checkmark & & &\checkmark     &   &   &  &  \\
            VOS~\cite{vos}    &\checkmark & & &\checkmark & & & & \\
            LogitNorm~\cite{logitnorm}    &\checkmark & & &\checkmark & & & & \\
            - \textbf{Adversarial} \\
            ACET~\cite{acet} &(\checkmark) &\checkmark & &\checkmark &(\checkmark) & & &   \\
            ATOM~\cite{atom} &(\checkmark) & & &\checkmark     &(\checkmark) & & &  \\
            - \textbf{Guaranteed} \\
            GOOD~\cite{good} &           & & &               &\checkmark  &\checkmark  &  &\checkmark    \\
            ProoD~\cite{prood}  &\checkmark & & &\checkmark &\checkmark &\checkmark & &\checkmark  \\
            DISTRO (Our)        &\checkmark &\checkmark &\checkmark &\checkmark &\checkmark &\checkmark &\checkmark &\checkmark \\
            \bottomrule
        \end{tabular}
    \end{adjustbox}
\end{table*}

In this study, we propose a novel technique for certifying OOD detection within the $\ell_2$-norm of the input sample, without requiring the use of binary discriminators or specific training.
This enables us to establish a guaranteed upper bound on the classifier's confidence within a defined region surrounding the input.
Unlike before, certified robust OOD detection can now be computed for standard OOD detection approaches.
Additionally, we incorporate a diffusion denoiser~\cite{nichol2021improved, dds}, which recovers the perturbed images and returns high quality denoised inputs.
This leads to better levels of both adversarial and certified robustness for ID and OOD data. This work and previous methods are compared in \autoref{tab:contribution}.

In summary, our contributions are:
\begin{itemize}
    \item A novel technique to robustly certify the confidence of any classifier within an $\ell_2$-norm on OOD data.
    This technique can be applied to any architecture and does not require additional components, even though it has higher computational costs compared to previous approaches.
    \item A method named DISTRO: \textbf{DI}ffusion denoised \textbf{S}moo\textbf{T}hing for \textbf{R}obust \textbf{O}OD detection.
    This method incorporates a diffusion denoiser model to improve the detection of adversarial and certified OOD samples, while providing high adversarial and certified accuracy for ID data.
\end{itemize}

\section{Background}\label{sec:background}

We define a \textit{hard} classifier as a function $f : \sR^d \to \set{Y}$ which maps input samples $x \in \sR^d$ to output $y \in \set{Y}$, where $\set{Y} = \{1, \dots, K\}$ is the discrete set of $K$ classes. 
Additionally, we introduce a \textit{soft} version $F : \sR^d \to \prob(\set{Y})$ of $f$, where $\prob(\set{Y})$ is the set of probability distributions over $\set{Y}$.
It is possible to convert any soft classifier $F$ into a hard classifier $f$ by mapping $f(x) = \argmax_{y\in\set{Y}} F(x)_y$.
Additionally, we define as $\set{N}(0, 1)$ the standard Gaussian distribution, as $\Phi(x)$ the Gaussian CDF and as $\Phi^{-1}(x)$ its inverse (or quantile).

\textbf{Robustness Certificates.}
Even though an adversarially-trained network is resilient to attacks created during training, it can still be susceptible to unseen new attacks.
To overcome this problem, certified defenses formally guarantee the stability of the prediction in a neighbourhood of the input. 
In other words, a neural network $f$ is certifiably robust for the input $x \in \mathbb{R}^d$, if the prediction for all perturbed versions $\tilde{x}$ remains unchanged such that $\norm{\tilde{x}-x}_p \leq \epsilon$, where $\norm{\cdot}_p$ is the $\ell_p$-norm around $x$ of size $\epsilon > 0$.

\textbf{Randomized Smoothing.} 
This robustness verification method \cite{randomized_smoothing} computes the $\ell_2$-norm certificates around an input sample $x$ by counting which class is most likely to be returned when $x$ is perturbed by isotropic Gaussian noise. 
Formally, given a \textit{soft} classifier $F$, randomized smoothing considers a \textit{smooth} version of $F$ defined as:
\begin{equation}\label{eq:smooth_classifier}
    G(x) \eqdef \E_{\delta \sim \set{N}(0, \sigma^2I)}\left[F(x + \delta)\right],
\end{equation}
where $\sigma > 0$ represents the standard deviation.
As previously, we define the hard version of $G(x)$ as $g(x) =\argmax_{y\in\set{Y}} G(x)_y$.
\citet{randomized_smoothing} demonstrated that $G$ is robust to perturbations of radius $R$, where the radius $R$ is defined as the difference in probabilities between the most likely class and the second most likely class.
A more general interpretation is given by ~\citet{yang2020randomized}.
\begin{restatable}{lemma}{randomizedsmoothing}[\citet{yang2020randomized}]\label{th:randomized_smoothing}
    Given a smoothed classifier $G$ defined as in \autoref{eq:smooth_classifier}, such that $G(x) = (G(x)_1, \dots, G(x)_K)$ is a vector of probabilities that $G$ assigns to each class $1, \dots, K$. 
    Suppose $G$ predicts class $c$ on input $x$, and the probability is $p = max_{y\in \set{Y}} G(x)_y > 1/2$, then $G$ continues to predict class $c$ when $x$ is perturbed by any $\delta$ with:
    \begin{equation*}
        \norm{\delta}_2 < \sigma \Phi^{-1}(p).
    \end{equation*}
\end{restatable}
%
One should consider $p$ as the probability that the smoothed classifier will assign to the predicted class rather than any other.
As a consequence, if $p > 1/2$, it will continue to do so even if the input is perturbed by Gaussian noise of magnitude smaller than the radius $R = \sigma \Phi^{-1}(p)$.

\citet{salman} show that randomized smoothing can postprocess the network to make it locally Lipschitz continuous.
The connection between randomized smoothing and Lipschitz continuity is provided in the following lemma, which offers an analytical form of the gradient of a smooth function.

\begin{restatable}{lemma}{lemmastein}[\citet{stein1981estimation}]\label{lemma:stein}
    Let $\sigma>0$, let $h:\R^d \to \R$ be measurable, and let $H(x) = \E_{\delta \sim \set{N}(0, \sigma^2 I)}[h(x+\delta)]$. Then $H$ is differentiable, and moreover:
    \begin{equation*}
        \nabla H(x) = \frac{1}{\sigma^2}\E_{\delta \sim \set{N}(0, \sigma^2 I)} \left[ \delta \cdot h(x + \delta) \right].
    \end{equation*}
\end{restatable}
The smoothed function $H$ is also known as the \textit{Weierstrass transform} of $h$, and a classical property of the
Weierstrass transform is its induced smoothness.

\textbf{Diffusion Denoised Smoothing.} 
In a nutshell, forward diffusion involves adding Gaussian noise to an image until it produces an isotropic Gaussian distribution with a large variance. 
Denoising diffusion probabilistic models work by learning how to reverse this process.
In formal terms, given an input sample $x \in \R^{d}$, a diffusion model selects a predetermined \textit{timestep} $t\in \mathbb{N}^+$  and samples a noisy image $x_t$ as follows:
\begin{equation}\label{eq:diffusion}
    x_t \eqdef \sqrt{\alpha_t} \cdot x + \sqrt{1 - \alpha_t} \cdot \set{N}(0, I),    
\end{equation}
where the amount of noise to be added to the image is determined by a constant called $\alpha_t$ derived from $t$.

As \citet{salman2020denoised} suggested, denoising Gaussian pertubed images leads to out-of-the-box certified robustness for plain models.
Following this trend~\citet{dds} make use of a diffusion model as one-shot denoiser achieving state-of-the-art performances.
The minor proposed adjustment held in the estimation of $t$, computed such that $\frac{1-\alpha_t}{\alpha_t} = \sigma^2$.
Additionally, the perturbed version $\tilde{x} = x + \delta$ is scaled by $\sqrt{\alpha_t}$, to match the noise model of \autoref{eq:diffusion}.
\section{Certified Robust OOD Detection}\label{sec:method}

This section explains how using local Lipschitz continuity, achieved through smoothing the classifier with Gaussian noise, can guarantee the detection of OOD samples within a $\ell_2$-sphere around the input.

\textbf{Preliminaries.} 
To determine how well a classifier distinguishes between ID and OOD samples, it is common to threshold the confidence level and to calculate the area under the receiver operating characteristic curve (AUROC or AUC).
Formally, let us consider a function\footnote{e.g. the Maximum Softmax Probability~\cite{msp}, or the Energy function~\cite{energy}.} $h\in\R^d \to \R$, the AUC is defined as:
\begin{equation*}
    \text{AUC}_h (\set{D}_{in}, \set{D}_{out}) = \E_{
    \begin{subarray}{l} x \sim \set{D}_{in}, \\ 
    z \sim \set{D}_{out}\end{subarray}} 
    \left[\mathbbm{1}_{h(x) > h(z)}\right],
\end{equation*}
where $\set{D}_{in}, \set{D}_{out}$ are ID and OOD data sets, respectively, and $\mathbbm{1}$ returns 1 if the argument is true and 0 otherwise.
A number of prior works~\cite{ccu, good, atom, prood} also investigated the worst-case AUC (WCAUC), which is defined as the lowest AUC attainable when every OOD sample is perturbed so that the highest level of confidence is achieved within a specific threat model.
Specifically, the WCAUC is defined as:
\begin{equation*}
    \text{WCAUC}_h (\set{D}_{in}, \set{D}_{out}) = \E_{
    \begin{subarray}{l} x \sim \set{D}_{in}, \\ 
    z \sim \set{D}_{out}\end{subarray}}
    \left[\mathbbm{1}_{h(x) > \underset{\norm{\tilde{z}-z}_p \leq \epsilon}{\max} h(\tilde{z}) } \right].
\end{equation*}

Due to the intractable nature of the maximization problem, we can compute upper or lower bounds only, i.e. $\underline{h}(z) \leq \max_{\norm{\tilde{z}-z}_p \leq \epsilon} h(\tilde{z}) \leq \bar{h}(z)$.
The lower bound $\underline{h}(z)$ is typically calculated using projected gradient methods~\citep{pgd, apgd} and named Adversarial AUC (AAUC) (upper bound of WCAUC).
In the context of $\ell_\infty$-norm, the upper bound $\bar{h}(z)$, called Guaranteed AUC (GAUC) (lower bound of WCAUC), is computed using IBP in \citet{good} and \citet{prood}.

Here, we propose a method for computing the upper bound of any classifier without the need for special training or modifications.
Thus, the main theorem for an $\ell_2$-norm robustly certified upper bound is stated.

\begin{restatable}{theorem}{upperbound}\label{th:upper_bound}
Let $F:\sR^d \to \prob{(\set{Y})}$ be any soft classifier and $G$ be its associated smooth classifier as defined in ~\autoref{eq:smooth_classifier}, with $\sigma > 0$.
If $p = \max_{y\in\set{Y}} G(x)_y > 1/2$, then, we have that:
\begin{equation}
    \max_{y\in \set{Y}} G(x+\delta)_y \leq \sqrt{\frac{2}{\pi}} \Phi^{-1}(p) + p,
\end{equation}
for every $\norm{\delta}_2 < \sigma \Phi^{-1}(p)$.
\end{restatable}
\begin{proof}
As a prerequisite to proving the theorem, we need to know the analytic form of the gradient of a smoothed function given in \autoref{lemma:stein}.
Let us consider the soft classifier $F(x):\R^d \to \prob(\set{Y})$, and its smooth version $G(x) = \E_{\delta \sim \set{N}(0, \sigma^2 I)} [F(x + \delta)]$, with $\sigma > 0$.
Since $F$ its a measurable function, we consider the \textit{Weierstrauss} transform of $F$ (which coincide with the \textit{smooth} version of $F$):
\begin{equation*}
    \E_{\delta \sim \set{N}(0, \sigma^2I)}\left[F(x+\delta)\right] = \left(F * \set{N}(0, \sigma^2I)\right)(x),
\end{equation*}
where $*$ denotes the convolution operator.
Thus, $G(x)$ is differentiable and from \autoref{lemma:stein} we have:
\begin{equation*}
    \nabla G(x) = \frac{1}{\sigma^2}\E_{\delta \sim \set{N}(0, \sigma^2 I)} \left[ \delta \cdot h(x + \delta) \right].
\end{equation*}
Since $F:\R^d \to [0, 1]$ and $\ell_2$ is self-dual, it is sufficient to show that the gradients of $G$ are bounded in $\ell_2$. 
From \autoref{lemma:stein}, for any unit vector $v\in \R^d$ we have that $|\langle v, \nabla G(x) \rangle|$ is equal to:
\begin{equation*}
    \begin{aligned}
    &\left\lvert \frac{1}{(2\pi \sigma^2)^{\nicefrac{d}{2}}} \int_{\sR^{d}} F(t) \left\langle v, \frac{t - x}{\sigma^2} \right\rangle e^{\left(-\frac{1}{2\sigma^2} \norm{x-t}_2^2 \right)} dt \right\rvert, \\
    &\leq \frac{1}{(2\pi \sigma^2)^{\nicefrac{d}{2}}} \int_{\sR^{d}} \left\lvert \left\langle v, \frac{t - x}{\sigma^2} \right\rangle \right\rvert e^{\left(-\frac{1}{2\sigma^2} \norm{x-t}_2^2 \right)} dt,  \\
    \end{aligned}
\end{equation*}
where we make use of the triangle inequality and know that $F$ is bounded by 1.
Given that projections of Gaussians are Gaussians and from the classical integration of the Gaussian density, we obtain:
\begin{equation*}
    \begin{aligned}
        &\frac{1}{(2\pi \sigma^2)^{\nicefrac{d}{2}}} \int_{\sR^{d}} \left\lvert \left\langle v, \frac{t - x}{\sigma^2} \right\rangle \right\rvert e^{\left(-\frac{1}{2\sigma^2} \norm{x-t}_2^2 \right)} dt, \\
        =& \frac{1}{\sigma^2}\E_{Z\sim \set{N}(0, \sigma^2)} \left[\lvert Z \rvert\right] = \sqrt{\frac{2}{\pi\sigma^2}},
    \end{aligned}
\end{equation*}
where we consider the supremum over all unit vectors $v$.
Since, we know that $G(x)$ is $\sqrt{\frac{2}{\pi\sigma^2}}$-Lipschitz in $\ell_2$, it is possible to use the Lipschitz constant to bound the difference between $G(x + \delta)$ and $G(x)$ for any value of $\delta$, with $\norm{\delta}_2 < \sigma \Phi^{-1}(p)$, where $p = \max_{y\in\set{Y}} G(x)_y$.
Formally:
\begin{equation*}
    |G(x + \delta)| \leq \sqrt{\frac{2}{\pi\sigma^2}} \norm{\delta}_2 + |G(x)|,
\end{equation*}
where we make use of the reverse triangle inequality.
Since $G(x):\sR^d \to [0, 1]$, we can assume $|G(x)|=G(x)$, and moreover:
\begin{equation*}
    \max_{y \in \set{Y}} G(x + \delta)_y \leq \sqrt{\frac{2}{\pi}}\Phi^{-1}(p) + \max_{y \in \set{Y}} G(x)_y.
\end{equation*}

\end{proof}
In other words, if the smooth classifier assigns the most likely class more than half the time, it is locally Lipschitz continuous in $x$, and its maximum prediction is bounded within a radius smaller than $R = \sqrt{\frac{2}{\pi}} \Phi^{-1}(p)$.

\subsubsection*{Discussion}

While this theorem provides some advantages, it is important to note a couple of its limitations.
One of the main limitations is that the upper bound of the smooth classifier $G$ only applies to $G$ and not to the original classifier $F$. 
As a result, the guarantee only applies to $G$, and its robustness at a given input point $x$ cannot be precisely evaluated or certified. 
To overcome this, Monte Carlo algorithms can be used to approximate these evaluations with high probability~\cite{randomized_smoothing}.

Another limitation is that the guarantees provided by this theorem are only probabilistic in practice. 
Therefore, a hypothesis test~\cite{hypothesis_test} should be used to avoid making predictions with low confidence. 
As with randomized smoothing~\cite{randomized_smoothing}, a large number of samples must be generated in order to achieve high levels of confidence in the certification radius. 
However, generating these samples can be computationally expensive for complex models.

Despite these limitations, the theorem provides a novel way of calculating the upper bound of any classifier, without the need for special training or modification.
Additionally, we provide a tighter certificate compared to previous approaches \cite{good, prood}, as they used IBP.
This can be useful for evaluating the certified robustness of a broader category of standard OOD detection methods as well as larger models, where IBP bounds explode in size and make them unusable~\cite{hurt_training}.

\section{DISTRO: DIffusion denoised SmooThing for Robust OOD detection}\label{sec:distro}

In this section, we present our method.
Essentially, it combines three techniques: (i) a diffusion denoiser, (ii) a standard OOD detector, and (iii) a certified binary discriminator.
Each component of this method is designed to overcome a specific problem of ordinary classifiers, as they are not robust to adversarial attacks, either ID or OOD, and do not detect OOD inputs well.

\tikzstyle{denoises} = [draw, fill=lightgray!20, text width=10em, text centered, minimum height=2em, rounded corners]
\tikzstyle{classifier}=[draw, fill=lightgray!20, text width=6em, text centered, minimum height=2em, rounded corners]
\tikzstyle{discriminator}=[draw, fill=lightgray!20, text width=6em, text centered, minimum height=2em, rounded corners]
\tikzstyle{ann} = [right]

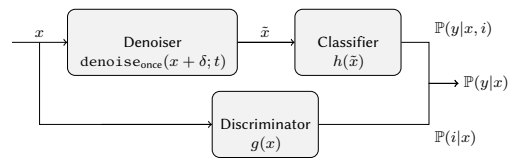
\begin{figure}[ht]
    \begin{adjustbox}{width=0.45\textwidth,center}
    \begin{tikzpicture}
        \node (denoise) [denoises] {\begin{center}
            Denoiser\\ 
            $\texttt{denoise}_{\text{once}}(x + \delta; t)$ \end{center}};

        \path (denoise.east)+(4, -0.75) node (output) [ann] {$\prob(y \lvert x)$};
        \path (denoise.east)+(2, 0) node (classifier) [classifier] 
        {\begin{center} Classifier\\ $h(\tilde{x})$\end{center}};
        
        \path (denoise.east)+(0.5, -1.5) node (discriminator) [discriminator] 
        {\begin{center} Discriminator \\ $g(x)$ \end{center}};

        \path [draw, ->] (denoise.west)+(-1, 0) -- node (input) [above] {$x$} (denoise.west);
        \draw[->] (denoise.east)+(1, 0) -- node (input_classifier) [above] {$\tilde{x}$} (denoise.east) -- (classifier.west);
        
        \draw[->] (denoise.west)+(-0.5, 0) |- (discriminator.west)+(-2, 0) --  (discriminator.west);
        
        \path [draw, ->] (classifier.east) -| ([xshift=5mm] classifier.east) node[above right] {$\prob(y \lvert x, i)$} |- (output.west);
        \path [draw, ->] (discriminator.east) -| ([xshift=20mm] discriminator.east) node[below right] {$\prob(i \lvert x)$} |- (output.west);

    \end{tikzpicture}
    \end{adjustbox}
    \caption{Overview of DISTRO.}
    \label{fig:overview}
\end{figure}

In \autoref{fig:overview}, we show an overview of DISTRO. 
First, a diffusion denoiser is employed before the classifier itself to provide robustness against ID attacks. 
As a result, adversarial noise introduced by the attack is mitigated by the denoiser.
This technique has already been proven to be very efficient and does not affect clean accuracy \cite{dds}. 

\begin{figure*}
\vspace{-0.5em}
    \begin{subfigure}{.45\textwidth}
        \centering
        \includegraphics[width=\textwidth]{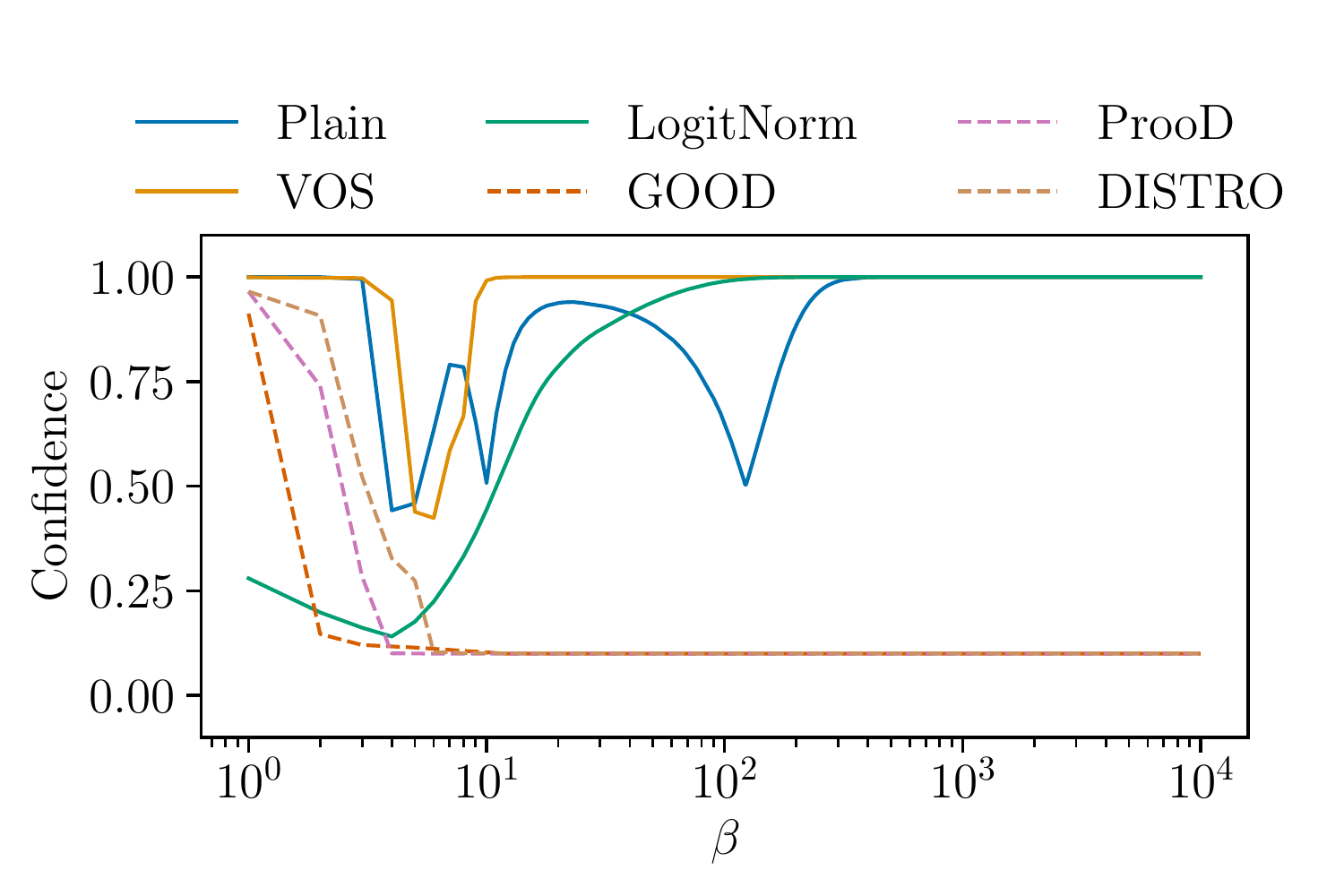}
        \caption{MSP}
        \label{fig:sub_confidence_msp}
    \end{subfigure}
    \hfill
    \begin{subfigure}{.46\textwidth}
        \centering
        \includegraphics[width=\textwidth]{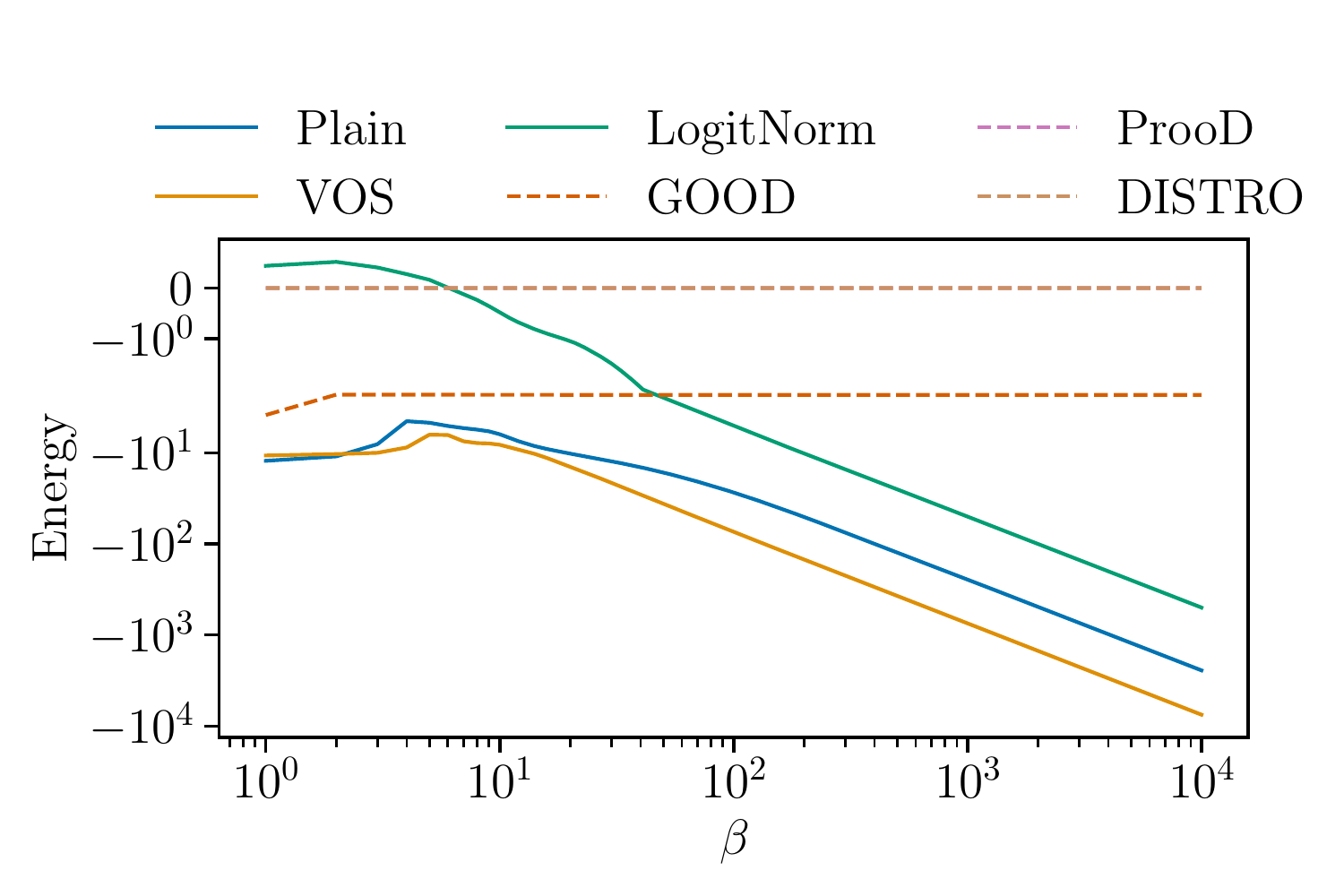}
        \caption{Energy}
        \label{fig:sub_confidence_energy}
    \end{subfigure}
    \vspace{-0.5em}
    \caption{Asymptotic confidence as: (a) MSP~\cite{msp} and (b) Energy~\cite{energy}, for several OOD detection models divided into two categories: \textit{standard} (continuous line) and \textit{guaranteed} (dashed line).}
    \label{fig:asymptotic_confidence}
\end{figure*}

Secondly, numerous post-hoc OOD detection methods exist. 
The most straightforward being MSP~\cite{msp}, which can be added to the image classifier without retraining or fine-tuning. 
Alternatively, standard OOD detection methods, such as OE~\cite{oe}, VOS~\cite{vos} or LogitNorm~\cite{logitnorm}, could also replace the classifier.
Thirdly, to make the model more robust to OOD adversarial attacks, we add a binary discriminator to the model that is trained to be certifiably robust against OOD attacks. 
Additionally, this discriminator is combined with the OOD detection method from (ii) which is necessary to have the property of asymptotic underconfidence for far-OOD inputs.


\textbf{Configuration.}
This method does not require any new technical knowledge. 
We begin by making the assumption that OOD samples are unrelated and thus maximally un-informative to the ID data.
Thus, for every class $y \in \set{Y}$, the conditional distribution on the input $x$ is given as:
\begin{equation}\label{eq:joint_prob}
    \prob(y|x) = \prob(y|x,i)\prob(i|x) + \frac{1}{K}(1-\prob(i|x)),
\end{equation}
where $\prob(i|x)$ is the conditional distribution representing the probability that $x$ is part of the ID, while $\prob(y|x,i)$ is the conditional distribution representing the ID.
Similarly to \citet{prood}, we assign independent models to each distribution:
\begin{itemize}
    \item $\prob(y|x,i) = h(\mathtt{denoise}_{\text{once}}(x+\delta; t))$, where $h:\R^d\to [0, 1]$ is the confidence of the main classifier $F(x)$, and $\tilde{x} = \mathtt{denoise}_{\text{once}}(x+\delta; t)$ represents one single step of denoising operation with $\delta \sim \set{N}(0, \sigma^2 I)$.
    \item $\prob(i|x) = \frac{1}{1+e^{-g(x)}}$, where $g:\R^d\to \R$ refers to a binary discriminator trained in a certified robust manner based on an $\ell_\infty$-threat model as in \citet{good, prood}.
\end{itemize}

As can be seen, the denoiser is the main addition. 
The one-step denoiser $\mathtt{denoise}_{\text{once}}$ estimates the fully denoised image $x$ from the current timestep $t$. 
Then it computes the average between the denoised image and the noisy image from the previous timestep.
As discussed in ~\citet{dds}, multiple applications of the denoiser will only destroy information about $x$.
Denoising with iterative steps essentially transfers the classification task to the denoiser, which can determine how the image should be filled.
For these reason, we apply only a single step of denoising.

\textbf{Asymptotic Underconfidence.}
Here, we show that by coupling a classifier trained to be OOD aware with a diffusion denoiser and running a certified discriminator in parallel, we can guarantee asymptotic underconfidence for data \textit{far enough} from the training distribution.

To obtain asymptotic underconfidence of the joint classifier, we consider $\prob(y|x,i) \leq 1$ and rewrite ~\autoref{eq:joint_prob} as follows:
\begin{equation}
    \prob(y|x) \leq \frac{K-1}{K} \prob(i|x)+\frac{1}{K}.
\end{equation}
Since the right term only depends on $\prob(i|x)$, we just need to assure that $\lim_{\beta \to \infty}\prob(i|\beta x) \to 0$.
If we employ a certified binary discriminator, trained with IBP on OOD data, as descibed in \citet{prood}, to compute $\prob(i|x)$, we achieve asymptotic underconfidence independently of the main classifier.
Readers are referred to \citet{prood} for a more detailed explanation.

\textbf{Empirical Evaluation.}
In \autoref{fig:asymptotic_confidence}, we show an empirical evaluation of the asymptotic confidence for standard and robust OOD detection methods\footnote{the models are described in \autoref{sec:experiments}.}.
In this test, we consider a single ID sample $x$ and multiply by a scalar $\beta$.
In \autoref{fig:sub_confidence_msp} we plot the MSP~\cite{msp} as confidence, while in \autoref{fig:sub_confidence_energy} we plot the Energy~\cite{energy} for increasing values of $\beta > 0$.
In the context of MSP, we observe that standard OOD detection methods are asymptotically overconfident, after a small drop, whereas certified methods such as GOOD~\cite{good}, ProoD~\cite{prood} and DISTRO converge to $1/K$.
On the other hand, for Energy as $\beta$ increases, VOS~\cite{vos}, LogitNorm~\cite{logitnorm}, and Plain models asymptotically decrease, whereas GOOD~\cite{good}, ProoD (Meinke at al., 2022), and DISTRO remain stable. 

As a result, underconfidence can be easily obtained when using an energy score instead of MSP, regardless of whether it is on a plain or OOD aware model. 
However, asymptotic underconfidence does not necessarily imply that the model will perform better in detecting OOD samples since all inputs are usually normalized to some range (e.g. [0, 1] or [-1, 1]). 
Thus the choice of MSP over the energy function is directly related to the possibility of certified robustness for OOD samples.
\section{Experiments}\label{sec:experiments}

In this section, DISTRO is evaluated for a variety of robust ID and OOD tests and is compared to previous approaches.
As baseline, we consider the pre-trained models\footnote{\href{https://github.com/AlexMeinke/Provable-OOD-Detection}{https://github.com/AlexMeinke/Provable-OOD-Detection}} from \citet{prood}.
The normal trained (\textbf{Plain}) and outlier exposure (\textbf{OE})~\cite{oe} models share the same ResNet18~\cite{resnet} architecture and hyperparameters as \textbf{ProoD}~\cite{prood}.
\textbf{GOOD}~\cite{good} uses a 'XL' convolutional neural network.
Additionally, we evaluate the pretrained DenseNet101~\cite{densenet} models for \textbf{ATOM}~\cite{atom} and \textbf{ACET}~\cite{acet}; and the standard OOD detection methods: \textbf{VOS}\footnote{\href{https://github.com/deeplearning-wisc/vos}{https://github.com/deeplearning-wisc/vos}}~\cite{vos} and \textbf{LogitNorm}\footnote{\href{https://github.com/hongxin001/logitnorm_ood}{https://github.com/hongxin001/logitnorm\_ood}}~\cite{logitnorm} with the pretrained WideResNet40~\cite{wideresnet} models provided in the respective works.
We consider \textbf{DDS}~\cite{dds} with a pre-trained diffusion model\footnote{\href{https://github.com/openai/improved-diffusion}{https://github.com/openai/improved-diffusion}} from \citet{nichol2021improved} in front of the OE classifier.
With \textbf{DISTRO}, we incorporate the same pre-trained diffusion model of DDS before the main classifier of ProoD, and maintain its discriminator.
The diffusion models have been used with the settings described in \citet{dds}.
In the context of $\ell_\infty$, we set $\sigma = \sqrt{d} \cdot \epsilon$.

We evaluate all methods on the standard datasets \texttt{CIFAR10/100}~\cite{cifar} as ID.
For the OOD detection evaluation we consider the following set of datasets: 
\texttt{CIFAR100/10}, \texttt{SVHN}~\cite{svhn}, LSUN~\cite{lsun} cropped (\texttt{LSUN\_CR}) and resized (\texttt{LSUN\_RS}),  TinyImageNet~\cite{tiny} cropped (\texttt{TinyImageNet\_CR}), \texttt{Textures}~\citep{textures} and synthetic (\texttt{Gaussian} and \texttt{Uniform}) noise distributions.
We use a random but fixed subset of 1000 images for all datasets considered as a test for OOD.
For ID, we consider the entire dataset.
We run all our experiments on a single NVIDIA A100. 

\subsection{In-Distribution Results}\label{sec:id-results}

Here, we compare clean, adversarial, and certified accuracy for ID samples.
Adversarial accuracy is evaluated with AutoAttack~\citep{apgd} for $\ell_\infty$-norm attacks of budget $\epsilon \in \{\nicefrac{2}{255}, \nicefrac{8}{255}\}$.
We ran the standard version of AutoAttack without additional hyper-parameters. 
Certified accuracy is evaluated for $\ell_2$-norm robustness of deviation $\sigma \in \{0.12, 0.25\}$.
To this end, random smoothing is performed on 10'000 Gaussian distributed samples around the input with a failure probability of $0.001$.
All $R>0$ are considered for the certified accuracy.
In the context of DISTRO and DDS we run 100 evaluation of the entire test set of \texttt{CIFAR10} to estimate the clean accuracy and report the average.
Further, we ran AutoAttack in both \textit{rand} and \textit{standard} modes, and considered the lowest results for DISTRO and DDS.

\begin{table}[htb]
\vspace{-0.5em}
    \centering
    \caption{\textbf{ID Accuracy}: Results of clean, adversarial and certified accuracy (\%) on the \texttt{CIFAR10} test set.
    The grayed-out models have an accuracy drop greater than $3\%$ relative to the model with the highest accuracy.}
    \label{tab:in-distribution}
    \begin{adjustbox}{width=0.5\textwidth,center}
        \begin{tabular}{llccccc}
            \toprule
            \multirow{2}{*}{Method} &\multirow{2}{*}{Clean} &\multicolumn{2}{c}{Adversarial ($\ell_\infty$)} &\multicolumn{2}{c}{Certified ($\ell_2$)} \\
            & &$\epsilon = \nicefrac{2}{255}$ &$\epsilon = \nicefrac{8}{255}$ &$\sigma=0.12$ &$\sigma = 0.25$ \\
            \midrule
            Plain$^*$       &95.01  &2.16   &0.00   &28.14  &14.17 \\
            OE$^*$          &95.53 &1.97   &0.00   &31.48  &10.88 \\
            VOS$^\dag$      &94.62  &2.24   &0.00   &13.13   &10.02       \\
            LogitNorm$^\ddag$  &94.48  &2.65   &0.00   &12.53  &10.25 \\
            \gray{ATOM$^*$}    &\gray{92.33}  &\gray{0.00}   &\gray{0.00}   &\gray{0.00}   &\gray{0.00}  \\
            \gray{ACET$^*$}    &\gray{91.49}  &\gray{69.01}  &\gray{6.04}   &\gray{57.13}  &\gray{12.48} \\
            \gray{GOOD$^*_{80}$} &\gray{90.13}  &\gray{11.65}  &\gray{0.23}   &\gray{17.33}  &\gray{10.31} \\
            ProoD$^*$ $\Delta=3$  &95.46  &2.69   &0.00   &33.92  &13.50 \\
            DDS                   &\textbf{95.55} &72.97 &24.09 &82.26 &64.58 \\
            DISTRO (our)          &95.47  &\textbf{73.34} &\textbf{27.14}  &\textbf{82.77}   &\textbf{65.63} \\
            \bottomrule
        \end{tabular}
    \end{adjustbox}
    \scriptsize{$*$ Pre-trained models from \citet{prood}, $\dagger$ Pre-trained from \citet{vos}, \\ $\ddag$ Pre-trained from \citet{logitnorm}.
    }
\vspace{-2em}
\end{table}

In \autoref{tab:in-distribution}, we show the results.
As expected, Plain and OE are not robust to adversarial attacks.
This applies to ProoD as well, since OE is its primary classifier.
Similarly, standard OOD detection methods, as LogitNorm and VOS, show poor robustness for ID data.
GOOD demonstrates better results than ProoD for adversarial attacks and worse in terms of certified accuracy.
Suprisingly, ACET reveals strong adversarial and certified accuracy despite of its reduced clean accuracy.
Meanwhile, ATOM results in zero for all tests since any slight perturbation of the input triggers the last neuron used for OOD detection.

\subsubsection*{Discussion}

It is clear that diffusion models can enhance adversarial and certified robustness while maintaining high clean accuracy.
As diffusion introduces variance into gradient estimators, standard attacks become much less effective.
Nevertheless, robustness accuracy of diffusion models varies over different runs for the same input, so it should be defined differently from deterministic accuracy, e.g. as expectation.
Luckily, one-shot diffusion introduces such a tiny variance that throughout a few of runs, our results were similar.

\subsection{Evaluation Metrics}

To discriminate between ID and OOD samples, we use the confidence of the classifier, i.e. MSP~\cite{msp}. 
Traditionally, the following metrics are used to evaluate the OOD detection performance: 
(i) false positive rate (FPR95) of OODs when ID samples have a 95\% true positive rate;
(ii) the area under the receiver operating characteristic curve (AUROC or AUC); and
(iii) the area under the precision-call curve (AUPR).
In order to determine robustness, we compare adversarial (AAUC, AAUPR, AFPR) and guaranteed (GAUC, GAUPR, GFPR) versions of the previous metrics.
For the adversarial metrics, we use the settings in \citet{prood} to ensure a fair comparison.

\begin{figure}[htb]
    \centering
    \includegraphics[width=0.45\textwidth]{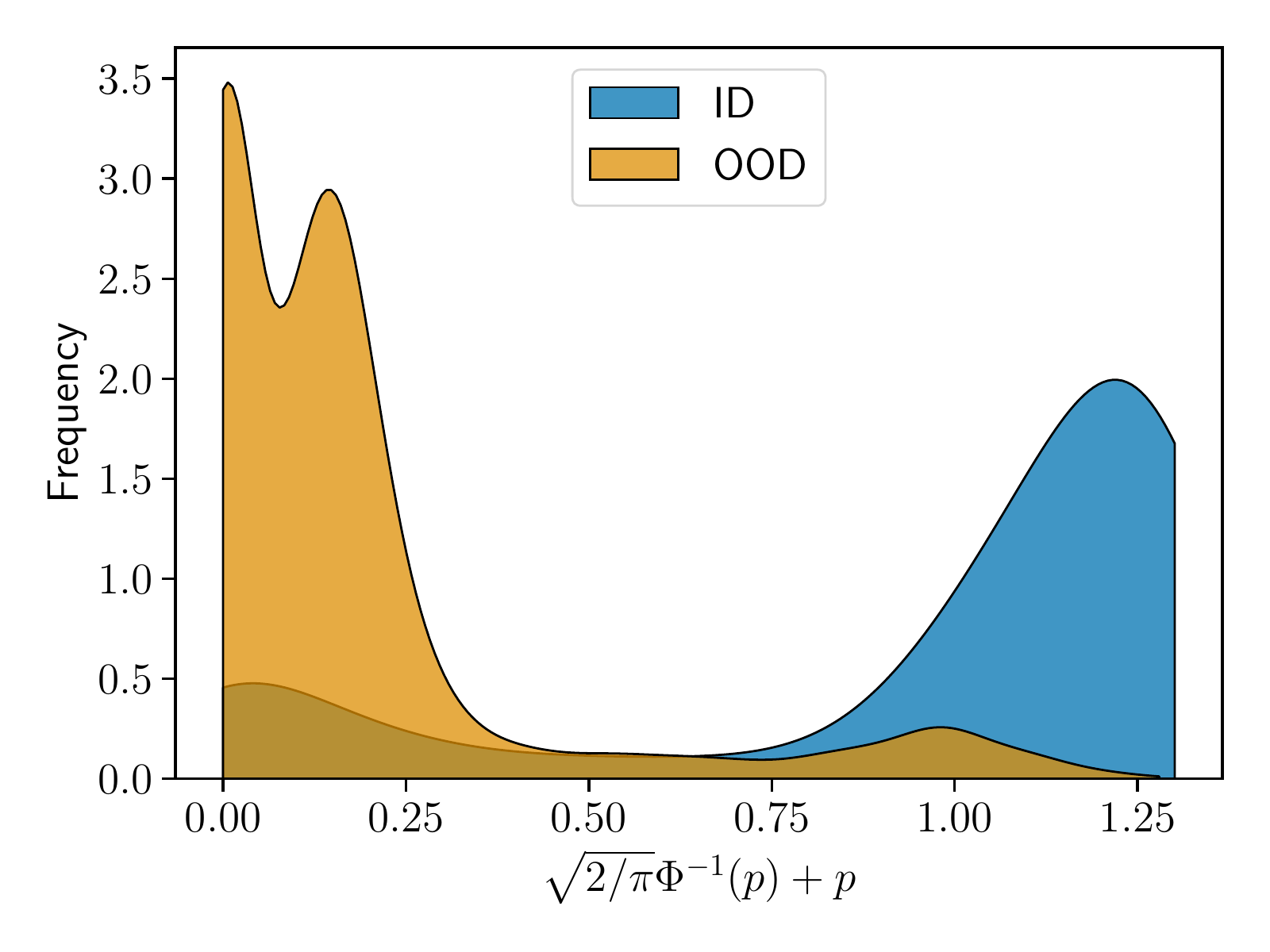}    
    \vspace{-1em}
    \caption{Kernel density estimation (bandwidth = 1) of the distribution of certified smooth ($\sigma = 0.12$) scores for DISTRO on ID (\texttt{CIFAR10}) and OOD (all other datasets) samples.}
    \label{fig:distribution_of_certified_scores}
\end{figure}

\textbf{Guaranteed.}
The guaranteed metrics (GAUC, GAUPR and GFPR) are computed for $\ell_2$ and $\ell_\infty$ norms robustness certificates.
Similarly to \citet{prood}, the $\ell_\infty$-norm is obtained with IBP only on OOD data.
On the other hand, the $\ell_2$-norm is computed with \autoref{th:upper_bound} on both ID and OOD data.
Similarly to \autoref{sec:id-results}, we sampled 10'000 Gaussian data points around the input with a deviation $\sigma = 0.12$. 
Since, the certified bound is only probabilistic in practice, we ran a binomial proportion confidence test~\cite{binomial_test} with failure probability of $0.001$.
We have assigned a score of 0 to all samples that fail to be certified, i.e. with $p \leq 1/2$. 
The Lipschitz continuity does not hold in the case of non-certified samples, therefore we are unable to bound the score.
To ensure a fair comparison, we decided to compute the $\ell_2$-norm GAUC on both ID and OOD datasets.

In \autoref{fig:distribution_of_certified_scores}, we plot the normalized frequency of occurrences of the certified upper bound ($\sqrt{2/\pi} \cdot \Phi^{-1}(p) + p$) for ID versus OOD data of DISTRO.
We observe that OOD data tend to peak close to zero, while ID data are spread out with larger values. 
This suggests that a large radius is more likely to be associated with ID data versus OOD samples.
As a result, robustly certifying the detection of OOD samples becomes more feasible.

\subsection{Out-Of-Distribution Results}

\begin{table*}[htb]
\vspace{-0.5em}
    \centering
    \caption{\textbf{Robust OOD detection.} We consider the following metrics: clean top-1 accuracy on \texttt{CIFAR10/100} test sets, clean AUC, guaranteed (GAUC), adversarial AUC (AAUC), clean AUPR, guaranteed AUPR (GAUPR), adversarial AUPR (AAUPR), clean FPR\@95\% (FPR), guaranteed FPR\@95\% (GFPR) and adversarial FPR\@95\% (AFPR). 
    Averaging was performed on a variety of OOD datasets. We consider MSP~\cite{msp} for all methods and metrics (with temperature $T=1$). The guaranteed $\ell_2$-norm is computed for $\sigma = 0.12$ for all $R>0$, while the adversarial and guaranteed $\ell_\infty$-norm are computed for $\epsilon = 0.01$. The grayed-out models have an accuracy drop greater than $3\%$ relative to the model with the highest accuracy. \textbf{Bold} numbers are superior results.}
    \label{tab:ood_average}
    \begin{adjustbox}{width=\textwidth,center}
    \begin{tabular}{lr|rrrr|rrrr|rrrr}
    \toprule
        ID: CIFAR10 &Acc. &AUC$\uparrow$ &\multicolumn{2}{c}{GAUC$\uparrow$} &AAUC$\uparrow$ &AUPR$\uparrow$ &\multicolumn{2}{c}{GAUPR$\uparrow$} &AAUPR$\uparrow$ &FPR$\downarrow$ &\multicolumn{2}{c}{GFPR$\downarrow$} &AFPR$\downarrow$ \\
        &   &   &\multicolumn{1}{c}{$\ell_2$} &\multicolumn{1}{c}{$\ell_\infty$} &\multicolumn{1}{c|}{$\ell_\infty$} &   &\multicolumn{1}{c}{$\ell_2$} &\multicolumn{1}{c}{$\ell_\infty$} &\multicolumn{1}{c|}{$\ell_\infty$} &   &\multicolumn{1}{c}{$\ell_2$} &\multicolumn{1}{c}{$\ell_\infty$} &\multicolumn{1}{c}{$\ell_\infty$} \\
        \midrule
        - \textbf{Standard} \\
        Plain$^*$   &95.01  &94.56  &48.86  &0.00 &24.52  &99.42     &60.05   &0.00    &82.30  &35.72  &100.0  &100.0     &96.72 \\
        OE$^*$      &\textbf{95.53} &\textbf{98.78}  &46.88 &0.00 &37.91  &\textbf{99.87}  &63.08 &0.00 &84.49  &\textbf{4.71}  &100.0   &100.0     &70.26 \\
        VOS$^\dag$ &94.62  &90.82 &30.13 &0.00 &20.62 &99.15 &41.62 &0.00 &81.80 &61.66 &94.10 &100.0 &100.0 \\
        LogitNorm$^\ddag$ &94.48  &96.71 &40.73 &0.00 &39.76 &99.64 &49.31 &0.00 &86.47 &13.95 &100.0 &100.0 &91.10 \\
        - \textbf{Adversarial} \\
        \gray{ACET$^*$} &\gray{91.48} &\gray{97.24} &\gray{60.21} &\gray{0.00} &\gray{93.01} &\gray{99.68} &\gray{76.22} &\gray{0.00} &\gray{99.16} &\gray{13.82} &\gray{95.65} &\gray{100.0} &\gray{32.15} \\
        \gray{ATOM$^*$} &\gray{92.33} &\gray{98.82} &\gray{97.15} &\gray{0.00} &\gray{44.65} &\gray{99.86} &\gray{95.51} &\gray{0.00} &\gray{85.74} &\gray{4.14}  &\gray{5.04}  &\gray{100.0} &\gray{62.65} \\
        \multicolumn{1}{l}{- \textbf{Guaranteed}} \\
        \gray{GOOD$^*_{80}$}   &\gray{90.13} &\gray{93.12} &\gray{36.45} &\gray{57.52} &\gray{78.11} &\gray{99.22}  &\gray{52.31} &\gray{89.54} &\gray{95.19} &\gray{30.00}  &\gray{100.0} &\gray{72.45} &\gray{47.55} \\
        ProoD$^*\Delta=3$ &95.46 &98.72 &52.36 &\textbf{59.56} &64.22 &\textbf{99.87}  &66.53 &\textbf{93.89} &94.52 &5.49  &100.0  &100.0 &86.49 \\
        DISTRO (our)     &95.47 &98.72 &\textbf{88.97} &59.53 &\textbf{83.24} &\textbf{99.87}  &\textbf{92.75}  &\textbf{93.89} &\textbf{97.32} &5.29 &\textbf{67.86} &100.0 &\textbf{34.56} \\
        \midrule
        ID: CIFAR100 &Acc. &AUC$\uparrow$ &\multicolumn{2}{c}{GAUC$\uparrow$} &AAUC$\uparrow$ &AUPR$\uparrow$ &\multicolumn{2}{c}{GAUPR$\uparrow$} &AAUPR$\uparrow$ &FPR$\downarrow$ &\multicolumn{2}{c}{GFPR$\downarrow$} &AFPR$\downarrow$ \\
        &   &   &\multicolumn{1}{c}{$\ell_2$} &\multicolumn{1}{c}{$\ell_\infty$} &\multicolumn{1}{c|}{$\ell_\infty$} &   &\multicolumn{1}{c}{$\ell_2$} &\multicolumn{1}{c}{$\ell_\infty$} &\multicolumn{1}{c|}{$\ell_\infty$} &   &\multicolumn{1}{c}{$\ell_2$} &\multicolumn{1}{c}{$\ell_\infty$} &\multicolumn{1}{c}{$\ell_\infty$} \\
        \midrule
        - \textbf{Standard} \\
        Plain$^*$   &\textbf{77.38} &81.60 &30.63 &0.00 &16.98 &97.84 &45.10 &0.00 &81.27 &82.52 &100.0 &100.0 &100.0 \\
        OE$^*$      &77.28 &90.41 &39.87 &0.00 &22.79 &98.90 &49.46 &0.00 &81.96 &47.49 &100.0 &100.0 &87.74 \\
        - \textbf{Adversarial} \\
        ACET$^*$ &74.47 &90.27 &36.36 &0.00 &27.68 &98.84 &43.50 &0.00 &82.60 &44.11 &\textbf{90.41} &100.0 &74.99 \\
        \gray{ATOM}$^*$ &\gray{71.73} &\gray{91.72} &\gray{84.38} &\gray{0.00} &\gray{31.52} &\gray{98.88} &\gray{79.95} &\gray{0.00} &\gray{83.36}      &\gray{30.81} &\gray{30.09}  &\gray{100.0} &\gray{73.69} \\
        - \textbf{Guaranteed} \\
        ProoD$^* \Delta=1$   &76.79 &\textbf{90.90} &42.83 &\textbf{37.67} &43.81 &\textbf{98.91} &50.90 &\textbf{89.66} &90.46 &42.12 &100.0 &100.0 &97.11 \\
        DISTRO (our) &76.78 &90.89 &\textbf{59.39} &37.53  &\textbf{62.77} &98.90 &\textbf{69.41} &89.63 &\textbf{93.59}  &\textbf{40.94} &100.0  &100.0 &\textbf{58.58} \\
        \bottomrule
    \end{tabular}
    \end{adjustbox}
    \footnotesize{$*$ Pre-trained models from \citet{prood}, $\dagger$ Pre-trained from \citet{vos}, $\ddag$ Pre-trained from \citet{logitnorm}.
    }
  \vspace{-1em}
\end{table*}

Here, we describe the results shown in \autoref{tab:ood_average}.
As previously, we grayed-out models with an accuracy drop greater than $3\%$ with respect to the model with highest accuracy.
The objective of this choice is to prioritize clean ID accuracy over all other metrics.
A comparison of the remaining metrics is then made on an equal basis.
Despite this, there is no direct comparison between the GAUC of $\ell_2$ and $\ell_\infty$ norms.
This is primarily due to the fact that the guaranteed upper bound of $\ell_\infty$ is computed only for OOD data, whereas $\ell_2$ is computed for both (ID \& OOD).
Additionally, we choose any radius $R>0$ for $\ell_2$, while for $\ell_\infty$, $\epsilon$ is fixed to $0.01$\footnote{This problem can be addressed by considering $R \geq \sqrt{d}\cdot \epsilon$.}.

We observe that the performances of LogitNorm and VOS on clean AUC, AUPR and FPR are suboptimal.
The reason for this is that we are evaluating MSP~\cite{msp} instead of the suggested normalization~\cite{logitnorm} and energy~\cite{vos} functions for LogitNorm and VOS, respectively.
To ensure a fair comparison we decided to standardize the output function across all models.
On \texttt{CIFAR100}, only the most effective methods of \texttt{CIFAR10} have been tested.

\textbf{Outcomes.}
In light of these considerations, we note that OE achieved the highest clean AUC, AUPR, and FPR.
In case of AAUC, ACET shows the best results for \texttt{CIFAR10}.
While ATOM achieves close to optimal performance for the guaranteed $\ell_2$-norm AUC, AUPR and FPR.
Both methods are trained adversarially on outliers, which makes them more robust on OOD data, but at the expense of a reduced clean accuracy.

\begin{wraptable}{l}{4.2cm}
    \centering
    \caption{\small Overall average between the metrics of \autoref{tab:ood_average} for \texttt{CIFAR10/100} (\texttt{C-10}, \texttt{C-100}).}
    \label{tab:ood_overall}
    \begin{adjustbox}{width=0.25\textwidth,center}
    \begin{tabular}{lrr}
    \toprule
    Method      &\multicolumn{2}{c}{Average} \\
                &\texttt{C-10} &\texttt{C-100} \\
    \midrule
     Plain              &44.02 &34.48 \\
     OE                 &50.12 &40.42 \\
     VOS                &38.60 &- \\
     LogitNorm          &46.31 &- \\
     ACET               &59.64 &41.86 \\
     ATOM               &64.79 &54.38 \\
     GOOD$_{80}$        &64.74 &- \\
     ProoD $\Delta=3$   &64.09 &52.51 \\
     DISTRO (our)       &\textbf{77.08} &\textbf{59.95} \\
     \bottomrule
    \end{tabular}
    \end{adjustbox}
\end{wraptable}

Similarly to the ID results, DISTRO demonstrates the potential benefits of diffusion models to augment the model robustness in terms of $\ell_2$-norm guaranteed and adversarial AUCs.
Although there is a slight decrease in $\ell_\infty$-norm GAUC, GAUPR and GFPR, which could likely be suppressed by fine tuning the classifier in conjunction with the denoiser.
In \autoref{tab:ood_overall}, we average all the metrics of \autoref{tab:ood_average} for \texttt{CIFAR10} (including clean ID accuracy).
Surprisingly, ATOM shows similar results as ProoD and GOOD. 
This can be related to the high certification radius obtained for GAUC of $\ell_2$-norm. 


\subsubsection{Similar Model Capacity}\label{app:standardized}

Here, we outline the configurations and results of \autoref{tab:ood_average_standard}. 
Each technique is evaluated using the same architecture, acknowledging that the results from \autoref{tab:ood_average} do not depend just on the performance of the method, but also on the robustness of the model and the specific OOD dataset utilized.
Therefore we retrain all presented methods using a ResNet18~\cite{resnet} architecture for \texttt{CIFAR10} and \texttt{CIFAR100} respectively. 
For methods that require an additional OOD dataset for training, such as OE~\cite{oe}, ACET~\cite{acet}, ATOM~\cite{atom}, ProoD~\cite{prood} and DISTRO, we use the same subset of \texttt{OpenImages}~\cite{openimages} containing 50'000 images.
Furthermore, we consider an input normalization of $0.5$ across all dimensions for both mean and standard deviation.
In addition, we attempt to be as minimally intrusive as possible when it comes to the default training procedure.

For Plain, OE and LogitNorm we run the implementation\footnote{\href{https://github.com/Jingkang50/OpenOOD}{https://github.com/Jingkang50/OpenOOD}} from \citet{yang2022openood} and leave the hyperparameters unchanged.
Similarly for ACET and ATOM, we only change the model architecture and normalization and run both  implementations from ATOM\footnote{\href{https://github.com/jfc43/informative-outlier-mining}{https://github.com/jfc43/informative-outlier-mining}}.
Lastly, we train ProoD\footnote{\href{https://github.com/AlexMeinke/Provable-OOD-Detection}{https://github.com/AlexMeinke/Provable-OOD-Detection}} from \citet{prood} using their training configuration files, where the discriminator is trained for 1000 epochs and the bias shift ($\Delta$) is 3/1 for \texttt{CIFAR10/100}, respectively.


\begin{table*}[htb]
\vspace{-0.5em}
    \centering
    \caption{\textbf{Robust OOD detection with ResNet18.} We consider the following metrics: clean top-1 accuracy on \texttt{CIFAR10/100} test sets, clean AUC, guaranteed (GAUC), adversarial AUC (AAUC), clean AUPR, guaranteed AUPR (GAUPR), adversarial AUPR (AAUPR), clean FPR\@95\% (FPR), guaranteed FPR\@95\% (GFPR) and adversarial FPR\@95\% (AFPR). 
    Averaging was performed on a variety of OOD datasets. We consider MSP~\cite{msp} for all methods and metrics (with temperature $T=1$). The guaranteed $\ell_2$-norm is computed for $\sigma = 0.12$ for all $R>0$, while the adversarial and guaranteed $\ell_\infty$-norm are computed for $\epsilon = 0.01$. The grayed-out models have an accuracy drop greater than $3\%$ relative to the model with the highest accuracy. \textbf{Bold} numbers are superior results.
    } 
    \label{tab:ood_average_standard}
    \begin{adjustbox}{width=\textwidth,center}
    \begin{tabular}{lr|rrrr|rrrr|rrrr}
    \toprule
        ID: CIFAR10 &Acc. &AUC$\uparrow$ &\multicolumn{2}{c}{GAUC$\uparrow$} &AAUC$\uparrow$ &AUPR$\uparrow$ &\multicolumn{2}{c}{GAUPR$\uparrow$} &AAUPR$\uparrow$ &FPR$\downarrow$ &\multicolumn{2}{c}{GFPR$\downarrow$} &AFPR$\downarrow$ \\
        &   &   &\multicolumn{1}{c}{$\ell_2$} &\multicolumn{1}{c}{$\ell_\infty$} &\multicolumn{1}{c|}{$\ell_\infty$} &   &\multicolumn{1}{c}{$\ell_2$} &\multicolumn{1}{c}{$\ell_\infty$} &\multicolumn{1}{c|}{$\ell_\infty$} &   &\multicolumn{1}{c}{$\ell_2$} &\multicolumn{1}{c}{$\ell_\infty$} &\multicolumn{1}{c}{$\ell_\infty$} \\
        \midrule
 		Plain &94.32 &92.28 &35.81 &0.00 &23.71 &99.00 &46.83 &0.00 &82.00 &40.21 &93.56 &100.0 &98.88 \\
		LogitNorm &94.71 &95.58 &34.19 &0.00 &35.00 &99.54 &49.63 &0.00 &85.14 &33.06 &95.12 &100.0 &92.20\\
		OE &92.41 &97.35 &50.56 &0.00 &37.95 &99.71 &62.25 &0.00 &85.51 &13.44 &100.0 &100.0 &74.91\\
		ACET &93.66 &\textbf{97.86} &37.45 &0.00 &65.21 &\textbf{99.75} &50.26 &0.00 &91.99 &8.94 &100.0 &100.0 &\textbf{50.29} \\
		\gray{ATOM} &\gray{91.90} &\gray{98.12} &\gray{97.98} &\gray{97.63} &\gray{62.79} &\gray{99.78} &\gray{98.16} &\gray{99.78} &\gray{91.49} &\gray{8.7} &\gray{9.42} &\gray{0.00} &\gray{51.56} \\
		ProoD &\textbf{95.20} &96.91 &44.95 &\textbf{63.44} &64.61 &99.63 &60.27 &\textbf{94.37} &94.42 &\textbf{16.03} &100.0 &\textbf{91.90} &78.22 \\
		DISTRO (our) &\textbf{95.20} &96.80 &\textbf{86.63} &59.86 &\textbf{71.70} &99.62 &\textbf{90.80} &93.78 &\textbf{95.72} &16.55 &\textbf{66.88} &99.96 &67.59\\
        \midrule
        ID: CIFAR100 &Acc. &AUC$\uparrow$ &\multicolumn{2}{c}{GAUC$\uparrow$} &AAUC$\uparrow$ &AUPR$\uparrow$ &\multicolumn{2}{c}{GAUPR$\uparrow$} &AAUPR$\uparrow$ &FPR$\downarrow$ &\multicolumn{2}{c}{GFPR$\downarrow$} &AFPR$\downarrow$ \\
        &   &   &\multicolumn{1}{c}{$\ell_2$} &\multicolumn{1}{c}{$\ell_\infty$} &\multicolumn{1}{c|}{$\ell_\infty$} &   &\multicolumn{1}{c}{$\ell_2$} &\multicolumn{1}{c}{$\ell_\infty$} &\multicolumn{1}{c|}{$\ell_\infty$} &   &\multicolumn{1}{c}{$\ell_2$} &\multicolumn{1}{c}{$\ell_\infty$} &\multicolumn{1}{c}{$\ell_\infty$} \\
        \midrule
		Plain &77.54 &84.50 &38.11 &0.00 &24.17 &98.16 &44.96 &0.00 &82.32 &67.61 &100.0 &100.0 &98.04 \\
		LogitNorm &76.25 &84.06 &40.93 &0.00 &47.64 &98.04 &46.80 &0.00 &87.25 &73.70 &100.0 &100.0 &87.98\\
		OE &75.84 &88.96 &38.90 &0.00 &17.90 &\textbf{98.72} &48.82 &0.00 &81.43 &49.61 &100.0 &100.0 &99.41\\
		\gray{ACET} &\gray{73.71} &\gray{95.65} &\gray{42.03} &\gray{0.00} &\gray{52.49} &\gray{99.44} &\gray{48.54} &\gray{0.00} &\gray{89.23} &\gray{13.96} &\gray{100.0} &\gray{100.0} &\gray{60.39} \\
		ProoD &\textbf{77.77} &\textbf{89.47} &40.72 &\textbf{37.68} &49.16 &98.66 &49.97 &\textbf{89.66} &91.08 &\textbf{40.44} &100.0 &100.0 &84.15 \\
		DISTRO (our) &77.73 &88.90 &\textbf{55.57} &29.71 &\textbf{51.89} &98.60 &\textbf{67.62} &87.44 &\textbf{91.71} &43.24 &100.0 &100.0 &\textbf{79.34} \\
        \bottomrule
    \end{tabular}
    \end{adjustbox}
  \vspace{-1em}
\end{table*}

\subsubsection*{Discussion}

It is evident that the $\ell_2$-norm GAUC (and GAUPR) diverge from zero when standard OOD detection models are considered. 
This illustrates the potential of the $\ell_2$-norm to provide certified OOD detection for any method and architecture.
Consequently, it facilitates the experimental evaluation of new robust OOD detection algorithms (both adversarial and certified).

As a side note, the one-shot denoiser appears to improve robustness certification metrics while not compromising clean metrics, such as AUC. 
In some cases, it also appears to be slightly better, even though the denoising process should produce images that are as similar as possible to those considered during training.
This is because a single shot of denoising does not compromise the OOD sample or generate an allucinated one.
Additionally, one-shot denoising introduces so little variance that in this benchmark, the results were similar across multiple runs.
\section{Conclusion}\label{sec:conclusion}

Current OOD robustness certification relies on external discriminators or loose certification mechanisms~\cite{prood}.
We propose an alternative using randomized smoothing~\cite{randomized_smoothing} for $\ell_2$-norm certificates, applicable to any classifier without specific requirements or training.
In comparison with previously proposed $\ell_\infty$-norm GAUC, standard approaches for OOD detection show non-zero results for guaranteed $\ell_2$-norm AUC and AUPR.
Unfortunately, a large number of samples derived around the input must be propagated through the network, increasing computational costs.
Additionally, we propose a method combining three techniques: diffusion denoising for noise removal, an OOD detection method, and a certified binary discriminator. 
This combination improves OOD robustness detection by around 13\%/5\% on CIFAR10/100 datasets compared to earlier approaches.

\bibliography{references/certificates, references/plain_ood, references/robust_ood, references/miscellaneous, references/datasets}



\end{document}